\documentclass{article}

\usepackage{microtype}
\usepackage{graphicx}
\usepackage{subfigure}
\usepackage{booktabs} 
\usepackage{multirow}
\usepackage{amsmath,amssymb,amsthm} 

\newtheorem{proposition}{Proposition}
\newcommand{\bsb}[1]{\boldsymbol{#1}}

\usepackage[accepted]{icml2021}

\icmltitlerunning{Commutative Lie Group VAE}

\begin{document}

\twocolumn[
\icmltitle{Commutative Lie Group VAE for Disentanglement Learning}




\begin{icmlauthorlist}
\icmlauthor{Xinqi Zhu}{syd}
\icmlauthor{Chang Xu}{syd}
\icmlauthor{Dacheng Tao}{jd}
\end{icmlauthorlist}

\icmlaffiliation{syd}{School of Computer Science, Faculty of Engineering, The University of Sydney, Australia}
\icmlaffiliation{jd}{JD Explore Academy, JD.com, China}

\icmlcorrespondingauthor{Xinqi Zhu}{xzhu7491@uni.sydney.edu.au}

\icmlkeywords{Machine Learning, ICML, Disentangled Representation Learning}

\vskip 0.3in
]



\printAffiliationsAndNotice{}  

\begin{abstract}
    We view disentanglement learning as discovering an underlying structure
    that equivariantly reflects the factorized variations shown in data.
    Traditionally, such a structure is fixed to be a vector space
    with data variations represented by translations along
    individual latent dimensions.
    We argue this simple structure is suboptimal since
    it requires the model to learn to discard the properties
    (e.g. different scales of changes, different levels of abstractness)
    of data variations, which is an extra work than equivariance learning.
    Instead, we propose to encode the data variations with groups,
    a structure not only can equivariantly represent variations,
    but can also be adaptively optimized to preserve the properties
    of data variations.
    Considering it is hard to conduct training on group structures,
    we focus on Lie groups and adopt a parameterization using
    Lie algebra. Based on the parameterization,
    some disentanglement learning constraints are naturally derived.
    A simple model named Commutative Lie Group VAE is introduced
    to realize the group-based disentanglement learning.
    Experiments show that our model can effectively learn
    disentangled representations without supervision,
    and can achieve state-of-the-art performance without extra constraints.
\end{abstract}

\section{Introduction}
\label{sec:introduction}

Equivariance has been widely considered as one of the most
important desiderata in representation learning
\cite{capsule1,CohenLieGroups,CohenGroupCNN,higgins2018definition}.
A representation is equivariant if the
transformations on the input data can be reflected
by transformations on the representation:
\begin{align}
    \sigma(g(x)) = g'\sigma(x), \label{eq:equivariance}
\end{align}
where $\sigma$ denotes the representation function, and $g, g'$ represent
the same transformation acting on the data space and
representation space respectively.
An invariant representation is achieved if $g'$ becomes
the identity transformation.

Unsupervised disentangled representation learning is to discover the
factorizable variations shown in data and encode them
with individual dimensions in representations \cite{Higgins2017betaVAELB,
Kim2018DisentanglingBF,chen2018isolating,
Burgess2018UnderstandingDI,Jeong2019LearningDA,
Zhao2017LearningHF,Li2020ProgressiveLA,Karras2020ASG}.
A commonly applied but less emphasized assumption is that
\emph{the disentangled representations are also equivariant},
because the semantic (or attribute) changes are reflected by
the shifting of different dimensions in the representations space.
The difficulty of this task is to learn the representation that
preserves such an equivariance \emph{without supervision}.

Existing unsupervised disentanglement methods usually learn the
equivariance mapping based on a fixed vector space.
We argue that this modeling
is suboptimal,
because it requires a model to do two tasks at the same time:
(1) to discover equivariance; and
(2) to ignore the \emph{properties} in data variations to obey a
fixed vector-space embedding.
Here the \emph{properties} in variations can include
different levels of abstractness (e.g. low-level vs semantic attributes),
different scales of variations (e.g. significant vs subtle changes),
certain structures (e.g. cyclic),
relation between variations (e.g. conditional relation),
etc.
Our hypothesis is that the equivariance is more likely to be
learned with an adaptive equivariant structure which is used to
fit the data variations.
By this means,
the model is relieved from doing a combined difficult learning task
and can thus focus on learning equivariance.
In this paper, group structures are adopted for this task,
and a conceptual illustration is shown in Fig. \ref{fig:intro_impress}.
There exist some previous works that also adopt group structures to
learn disentangled representations
\cite{higgins2018definition,NEURIPS2019_36e729ec,
Quessard2020LearningGS,painter2020linear}. However, these methods
use predefined (known) group structures,
which are neither adaptive nor generalizable.
Additionally, these models cannot be learned without supervision.
To the best of our knowledge, this is the first work to learn
unsupervised disentangled representations based
on adaptive group structures, which can be successfully applied
to complex datasets.

\begin{figure}[t]
\begin{center}
   \includegraphics[width=\linewidth]{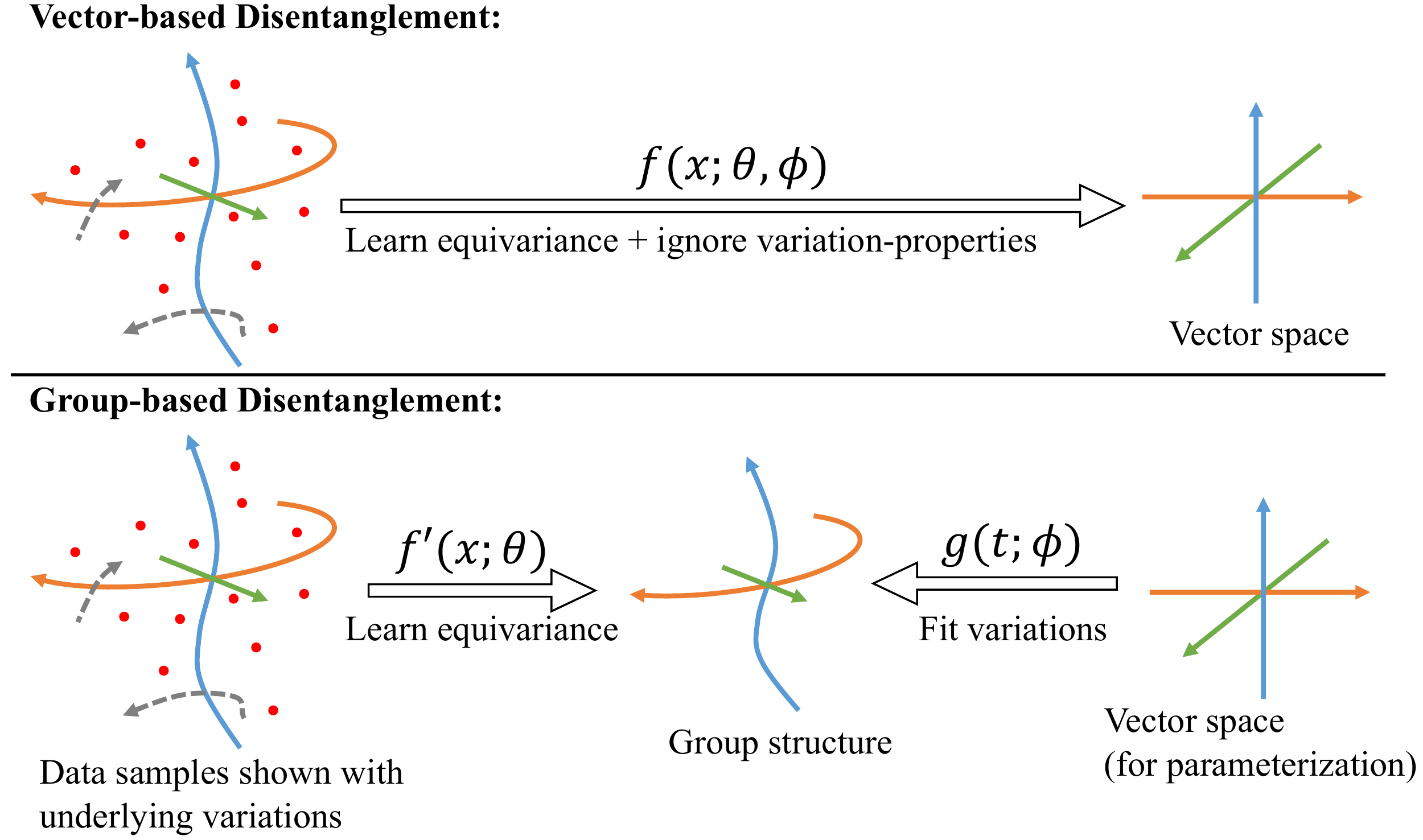}
\end{center}
    \caption{The upper part shows the classical disentanglement by learning
    equivariance with a vector space. On the left we show data samples
    in red points to be the static observations
    of some underlying variations.
    The underlying variations are shown as arrows in different colors with
    grey dash-arrows denoting noise that cannot be disentangled.
    The shape and length of arrows represent the properties of variations.
    An encoder $f$ is trained to simultaneously
    learn equivariance ($\theta$) and
    ignore variation-properties ($\phi$).
    The lower part shows our proposed group-based disentanglement.
    Our framework separates the whole work into two parts by
    training an adaptive group structure to fit the variations ($g$)
    and an encoder to focus on learning equivariance ($f'$).}
\label{fig:intro_impress}
    \vspace{-10pt}
\end{figure}

In this paper,
we propose to use groups \emph{as} representations to
achieve disentanglement without supervision.
We focus on Lie groups for modeling continuous variations
in data, and adopt the Lie algebra parameterization to
enable practical training.
Based on the parameterization,
decomposition constraints like one-parameter
subgroup decomposition and Hessian
Penalty can be naturally derived to encourage disentanglement.
As for the realization, we first introduce a simple variant
of VAE called bottleneck-VAE, which is constructed from a new lower bound
of $\text{log}\,p(x)$.
Based on the bottleneck-VAE,
the proposed Commutative Lie Group VAE can be naturally implemented
by incorporating Commutative Lie Group constraints.
Our models are validated in the unsupervised
disentanglement learning setting on various datasets.
Without extra disentanglement constraints (like statistical independence),
our group-based model achieves state-of-the-art
on DSprites and 3DShapes datasets.

\section{Related Work}
\label{sec:related_work}
\textbf{Disentanglement Learning.}
Supervised disentanglement learning has been mainly tackled as
conditional generation based on labeled-attributes \cite{Reed2014LearningTD,
Kingma2014SemisupervisedLW,Dosovitskiy2014LearningTG,Kulkarni2015DeepCI,
Yan2016Attribute2ImageCI,Lample2017FaderNM}.
After the introduction of InfoGAN \cite{Chen2016InfoGANIR} and
$\beta$-VAE \cite{Higgins2017betaVAELB},
the task of unsupervised disentanglement
learning has gained increasingly interest in recent years.
Most of the learning models are based on the VAE framework,
imposing constraints to learn continuous latent variables
like statistical independence \cite{Burgess2018UnderstandingDI,
Kumar2017VariationalIO,Kim2018DisentanglingBF,chen2018isolating},
hierarchical biases \cite{Chen2016InfoGANIR,Li2020ProgressiveLA}.
Some other models focus on adapting discrete variables into the
disentanglement learning \cite{Dupont2018LearningDJ,Jeong2019LearningDA}.
Another branch of methods is based on InfoGAN, modeling
the informativeness between the latent codes and the images,
e.g. IB-GAN \cite{Jeon2018IBGANDR}, InfoGAN-CR \cite{Lin2019InfoGANCRDG},
VPGAN \cite{VPdis_eccv20}, and PS-SC GAN \cite{Xinqi_cvpr21}.
These models all try to capture data variations with individual
latent variables in a vector space while we propose to
leverage an adaptive group structure to achieve boosted disentanglement
from a novel perspective.

\textbf{Symmetry-Based Disentanglement Learning.}
In addition to the common definitions of disentanglement learning
\cite{Bengio2012RepresentationLA,Eastwood2018AFF,Do2020Theory},
Higgins et~al. \cite{higgins2018definition}
propose another formal definition of disentanglement
learning based on group theory.
In \cite{NEURIPS2019_36e729ec,Quessard2020LearningGS}, concrete models are
proposed to learn
a symmetry-based representation, but both methods rely on the
paired data with action labels revealed to be trained.
In \cite{painter2020linear}, a reinforcement learning method
is incorporated to estimate the actions, but paired data samples of
elemental transformations are still required for training.
These models indeed attempt to capture symmetry-variations with groups,
but the group structures are predefined and the ground-truth factorization
is given, leading to the inability for unsupervised disentanglement learning.

\textbf{Group-Equivariant Convolutions.}
Inspired by the successful application of translation equivariance in
convolutional neural networks \cite{6795724}, there have been
a large number of works trying to bring other symmetry groups
into convolutional neural networks to improve data efficiency and
generalization, e.g. planar rotations
\cite{CohenGroupCNN,DielemanGroup,8100241,hoogeboom2018hexaconv}.
spherical rotations \cite{SphericalCNN,Worrall_2018_ECCV},
scaling \cite{NEURIPS2019_f04cd739,Sosnovik2020Scale-Equivariant},
general groups \cite{Bekkers2020B-Spline},
and groups on other data structures
\cite{finzi2020generalizing,fuchs2020se3transformers}.
Unlike our work, these works learn representations with certain symmetries
that are predefined and usually semantic-agnostic to improve
the data efficiency and generalization in neural networks.
On the contrary, we focus on \emph{discovering} variations that
can be equivariantly represented and disentangled, with group structures
unknown and representing semantic variations.

\section{Preliminaries of Groups}
\label{sec:preliminaries}
This paper adopts some essential concepts from
group theory which we exhibit here.

\textbf{Group.} A group is a set $G$ with a binary operation $\circ$
being the group multiplication. A group satisfies the following axioms:
\emph{Closure:} For all $h, g \in G$ we have $h \circ g \in G$;
\emph{Identity:} There exists an identity element $e \in G$ such that
$\forall g \in G, e \circ g = g \circ e = g$;
\emph{Inverse:} For each $g \in G$, there exists an inverse element
$g^{-1} \in G$ such that $g \circ g^{-1} = g^{-1} \circ g = e$;
\emph{Associativity:} For all $g,h,i \in G$, we have
$(g \circ h) \circ i = g \circ (h \circ i)$.
In this paper, we consider
matrix groups under matrix multiplication
(a subgroup of the general linear group $GL(V)$, where $V$ denotes
the vector space on which the matrix group is acting).

\textbf{Lie Group and Lie algebra.}
A Lie group is a group with a continuous (and smooth)
structure. In this paper, we consider matrix Lie groups, which are
Lie groups realized as groups of matrices.
Every Lie group is associated with a Lie algebra $\mathfrak{g}$, a vector
space that is the tangent space of the Lie group at the identity element.
A Lie algebra can be parameterized with a basis $\{A_i\}_{i=1}^m$,
where every element in $\mathfrak{g}$ is written as
$A = t_1 A_1 + ... + t_m A_m$ with $t_i$ being the coordinates.
Elements in a Lie algebra can be mapped to the corresponding Lie group
using the matrix exponential map
$\text{exp}: \mathfrak{g} \rightarrow G$.

\textbf{One-parameter subgroup.}
A one-parameter subgroup is (the image of) a
function $h: \mathbb{R} \rightarrow GL(V)$
if: (1) $h$ is continuous; (2) $h(0) = e$; (3) $h(t + s) = h(t)\circ h(s)$
for all $t, s \in \mathbb{R}$.
A Lie group with a 1-dim Lie algebra ($\{A_i\}_{i=1}^{m} = \{A\}$)
is a one-parameter subgroup: $h = \text{exp}(tA), t \in \mathbb{R}$.

\textbf{Group representation.} In group theory, group representations describe
abstract groups as linear transformations of vector spaces.
However, since we restrict our attention on matrix Lie groups and
also due to its potential confusion with the term \emph{representation}
in machine learning, we discard its group-theory definition in this paper.
Instead, we use the phrase \emph{group representation}
to describe a representation that is learned with a group structure.

\section{Method}
We propose group representations
for disentanglement learning in Sec. \ref{sec:group_as_code}.
To enable practical learning on group representations,
we introduce the Lie Algebra parameterization in Sec. \ref{sec:lie_param}.
In Sec. \ref{sec:lie_vae}, we first introduce a variant of VAE
called bottleneck-VAE, and then introduce our Commutative Lie Group VAE by
integrating the Commutative Lie Group constraints.

\subsection{Group Representation}
\label{sec:group_as_code}
Generally in an equivariance representation (as in Eq. \ref{eq:equivariance}),
there exist two separate components to be learned: the group
(of transformations) and the representation (of data).
This is not desirable since usually we need paired data to train the
group as we need supervision for transformations
\cite{NEURIPS2019_36e729ec,Quessard2020LearningGS}.
However, if all the variations shown in a dataset can be assumed
to be represented by a
group (which is the case in disentanglement learning),
then we can assume there exists a canonical data point $x_0$ such that
every other data point $x$ is transformed by a group element
$g_{0\rightarrow x}$ on the canonical data point:
\begin{align}
    \sigma(x) = \sigma(g_{0\rightarrow x}(x_0)) =
    g'_{0\rightarrow x}\sigma(x_0).
\end{align}
Now we can see the representation structure is actually determined by the group
since the group defines the \emph{relation} between samples
in the latent space while the absolute embedding position
of a single data point is not important.
It is thus reasonable to assign a fixed representation to the canonical
data point so that only the group structure is learned.
The key step is that we choose the group identity $e$ as the canonical
representation so that every sample has a \emph{group representation}:
\begin{align}
    \sigma(x) = g'_{0\rightarrow x}\sigma(x_0)
    = g'_{0\rightarrow x}e = g'_{0\rightarrow x},
\end{align}
and the samples are embedded on a group structure.
This group can now be learned with static observations since
each observation represents a transformation from the canonical data point.
Note that this group representation is similar to the pose representation used
in \cite{capsule1} where the pose of a capsule
is represented by a matrix which specifies the transformation
between the canonical entity and the actual instantiation.

Unfortunately, though we can assume such
a group representation exists, it is not
clear how it can be learned in practice.
It is also unclear how data samples can be mapped onto a group,
how sampling can be conducted on a group,
and how optimization can be implemented on a group.
In this paper, we shrink our concentration on
continuous groups (Lie groups) since
(1) most attributes in data consist of continuous variations,
and (2) it is easier to be parameterized (and thus learned).
We refer this type of representations as Lie group representations.

\subsection{Decomposition with Lie Algebra Parameterization}
\label{sec:lie_param}
We parameterize Lie groups in this paper by a
basis $\{A_i\}_{i=1}^m$ in the Lie algebra:
\begin{align}
    g(t) &= \text{exp}(A(t)), g \in G, A \in \mathfrak{g}, \nonumber \\
    A(t) &= t_1 A_1 + t_2 A_2 + ... + t_{m} A_{m}, \forall t_i \in \mathbb{R},
    \label{eq:lie_alg_def}
\end{align}
where $\mathfrak{g}$ is the Lie algebra of Lie group $G$, and
$\text{exp}(\cdot)$ is the matrix exponential map which maps
an element in a Lie algebra to the corresponding Lie group.
$t = (t_1, t_2, ..., t_{m})$ represents the coordinates in $\mathfrak{g}$
of a data sample.
When $t = \bsb{0}$, the corresponding element on the group
is $e = \text{exp}(\bsb{0})$, the identity.
The Lie algebra is a vector space
thus enables training with general optimization methods like SGD.
In our implementation, the
basis $\{A_i\}_{i=1}^m$ is optimized (as weights) to find
an adaptive group structure, and every data sample can thus be
identified by the coordinates $t$ in the Lie algebra.
This also enables data sampling since we can attach prior distributions
on the coordinates $t$ so as to simulate a distribution on the group structure.

This Lie algebra parameterization without any constraints
cannot guarantee a group structure to be decomposed into subgroups
with each subgroup independently parameterized
by a single coordinate (homomorphism between the group and
the coordinate space), e.g.
$\text{exp}(t_1 A_1 + t_2 A_2) \neq \text{exp}(t_1 A_1) \text{exp}(t_2 A_2)$.
This is not desirable since we would like a dimension
in the coordinate $t_i$ to identify a single subgroup
$\text{exp}(t_i A_i)$ and thus further represent a single
variation in the data space.
If this is not satisfied, the disentanglement is not
achieved since the data variations are not encoded into
separate subspaces ($t_1, t_2, ..., t_m$).
In the next two paragraphs we discuss two options to solve this problem.

\textbf{One-parameter subgroup decomposition constraint.}
Based on our Lie algebra parameterization (Eq. \ref{eq:lie_alg_def}),
we have the following proposition to decompose
a Lie group into one-parameter subgroups:
\begin{proposition}
    \label{prop:decomp}
    If $A_i A_j = A_j A_i, \forall i,j$, then
    \begin{align}
        &\emph{exp}(t_1 A_1 + t_2 A_2 + ... t_m A_m) \nonumber \\
        &= \emph{exp}(t_1 A_1)\emph{exp}(t_2 A_2)...\emph{exp}(t_m A_m) \\
        &= \prod_{\emph{perm}(i)} \emph{exp}(t_i A_i)
        \label{eq:decomp_constraint}.
    \end{align}
\end{proposition}

\begin{proof}
    See Appendix 1.
\end{proof}
Eq. \ref{eq:decomp_constraint} means the equation holds for
any permutation of the index $i$.
This decomposition ensures that the group structure is decomposed into subgroups
with each one parameterized by a single coordinate $t_i$.
Via this constraint, the data variations can be considered as disentangled
into individual latent dimensions $t_i$'s
if an equivariance between the data variations and the group structure
is learned.
Since this decomposition holds for any permutation of the subgroups,
the order of the subgroups would not influence the composition results
thus the original group becomes commutative in terms of the subgroups.
This shows a limitation of our method where it cannot disentangle
variations that cannot be equivariantly represented by a commutative
Lie group, e.g. 3D rotation decomposition along three orthogonal axes.

\textbf{Hessian Penalty constraint.}
Besides learning disentanglement via enforcing the subgroup decomposition,
we can also incorporate other useful disentanglement
constraints like Hessian Penalty
\cite{peebles2020hessian} to the group representation.
The Hessian Penalty assumes that the Hessian matrix with respect to a
disentangled representation is always zero since the
variation controlled by a dimension should not be a function
of another dimension (independent):
\begin{align}
    H_{ij} = \frac{\partial^2 f(z)}{\partial z_i \partial z_j} =
    \frac{\partial}{\partial z_j}
    \Big(\frac{\partial f(z)}{\partial z_i}\Big) = 0,
\end{align}
where $z$ is the disentangled representation,
and $f(\cdot)$ is a function of $z$.
Our Lie algebra parameterization is compatible with this
constraint, and we have the following proposition:
\begin{proposition}
    \label{prop:hessian}
    If $A_i A_j = 0, \forall i \neq j$, then
    \begin{align}
        H_{ij} = \frac{\partial^2 g(t)}{\partial t_i \partial t_j} = 0,
        \label{eq:hessian_constraint}
    \end{align}
\end{proposition}
where $g$ is the map defined in Eq. \ref{eq:lie_alg_def}.
\begin{proof}
    See Appendix 2.
\end{proof}
Note that this is a more strict constraint than Proposition
\ref{prop:decomp} since $A_i A_j = A_j A_i$ is implied by $A_i A_j = 0$,
which also enforces the commutative group decomposition.
This constraint further ensures that the dynamics caused
by a subgroup is not affected by other subgroups
at the group representation level (independent).
Different from the original Hessian Penalty paper \cite{peebles2020hessian}
where the constraint is implemented on multiple feature maps
with an unbiased stochastic approximator,
our method penalizes the model only on the Lie group structure
(using the Lie algebra basis), which is a different and a simpler
implementation.

In summary, the Lie algebra parameterization enables practical learning
by converting an optimization problem on groups to vector spaces.
It also enables new constraints for encouraging disentanglement
by enforcing commutative group decomposition and Hessian penalty.

\begin{figure*}[t]
\begin{center}
   \includegraphics[width=\linewidth]{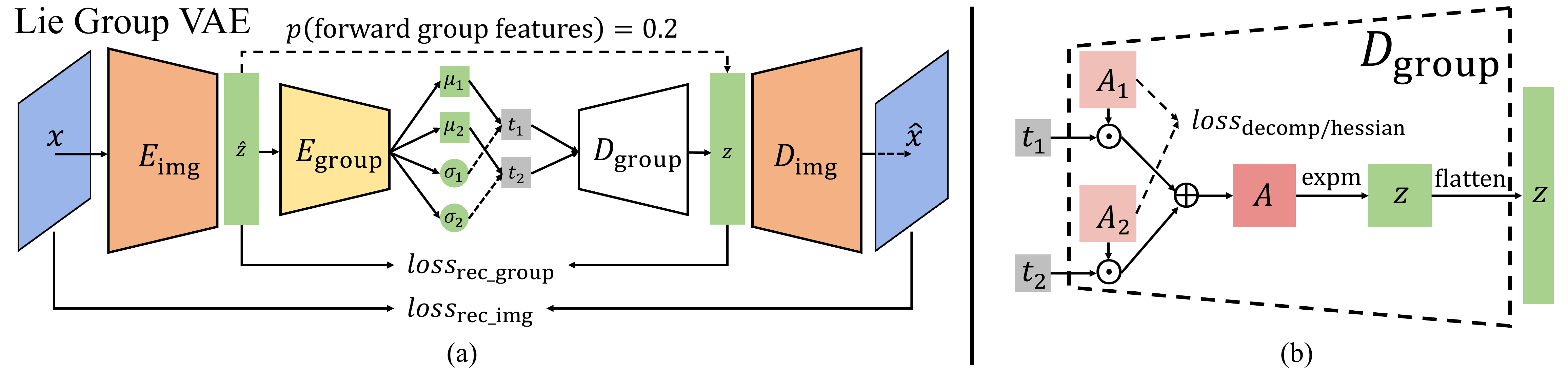}
\end{center}
    \caption{(a) The overall architecture of our proposed Lie Group VAE
    (bottleneck-VAE if $D_\text{group}$ is a general layer;
    Commutative Lie Group VAE if constraints from Proposition 1 and 2
    are applied).
    An input image is fed into the encoder $E_\text{img}$ to
    obtain the encoded group representation $\hat{z}$.
    The representation fed through a group encoder $E_\text{group}$
    (a two-layer MLP) to obtain the Lie algebra coordinates $t$
    using reparameterization trick.
    The coordinates are then fed through a group decoder (exponential
    mapping layer) to obtain the group representation $z$, which is
    then fed through a decoder to reconstruct the image.
    (b) Detailed illustration of the exponential mapping
    layer $D_\text{group}$. For each
    coordinate scalar $t_i$ we learn a Lie algebra basis $A_i$ (a matrix),
    where constraints from Proposition
    \ref{prop:decomp} and \ref{prop:hessian} can
    be applied as regularizations. $\odot$ denotes broadcast multiplication,
    $\oplus$ sums up the matrices, and $\text{expm}$ is the
    matrix exponential function. The obtained
    group representation is flattened to a vector before output.}
\label{fig:lie_vae}
\end{figure*}

\subsection{Commutative Lie Group VAE}
\label{sec:lie_vae}
In this section, we present a simple model to learn disentangled
representations using the group-related
techniques proposed in the last section.
We first introduce a VAE-variant (see Appendix 3 for an introduction
of VAE) called bottleneck-VAE which forces
a layer of feature in the encoder to match
a layer of feature in the decoder. The model is a realization
of the maximization of a lower bound of $\text{log}\, p(x)$,
which is presented in the following proposition:
\begin{proposition}
    Suppose two latent variables $z$ and $t$ are used to model the
    log-likelihood of data $x$, then we have:
    \begin{align}
        &\emph{log}\,p(x) \geq \mathcal{L}_{\text{bottleneck}}(x, z, t)
        \nonumber \\
        &= \mathbb{E}_{q(z|x)}\mathbb{E}_{q(t|x, z)}
        \emph{log}\,p(x,z|t) \nonumber\\
        &\quad- \mathbb{E}_{q(z|x)}KL(q(t|x,z)||p(t))
        - \mathbb{E}_{q(z|x)}\emph{log}\,q(z|x) \label{eq:pre_our_elbo}\\
        &= \mathbb{E}_{q(z|x)q(t|z)}
        \emph{log}\,p(x|z)p(z|t) \nonumber\\
        &\quad- \mathbb{E}_{q(z|x)}KL(q(t|z)||p(t))
        - \mathbb{E}_{q(z|x)}\emph{log}\,q(z|x) \label{eq:our_elbo},
    \end{align}
    where Eq. \ref{eq:our_elbo} holds because
    we assume Markov property: $q(t|z) = q(t|x,z)$, $p(x|z,t) = p(x|z)$.
\end{proposition}
\begin{proof}
    See Appendix 4.
\end{proof}
The variable $z$ is the feature to be shared in the encoder and the decoder.
In practice, we model $p(z|t)$, $p(x|z)$, $q(z|x)$ to be deterministic
networks ($D_\text{group}$, $D_\text{img}$, $E_\text{img}$ in Fig.
\ref{fig:lie_vae} (a)) while $q(t|z)$ to be a stochastic network
($E_\text{group}$ in Fig. \ref{fig:lie_vae} (a)).
The first term in Eq. \ref{eq:our_elbo} is implemented as two reconstruction
losses on $x$ and $z$ respectively ($loss_\text{rec\_img}$
and $loss_\text{rec\_group}$ in Fig. \ref{fig:lie_vae} (a)),
while the second term (KL-divergence)
is implemented the same as in a standard VAE. The third term is the
entropy of $z$ when $x$ is fixed, which is a constant since we model $q(z|x)$
to be a deterministic network, and is not implemented in practice.
We forward the output feature of the encoder $E_\text{img}$ directly to
the decoder $D_\text{img}$ at the probability of 0.2. This forces
that both features ($z$ and $\hat{z}$) are trained to
reconstruct the same image so that
they are encouraged to represent the same (high-level) information.

The Lie group structure is imposed on the (shared) feature $z$
in the bottleneck-VAE, and it is easily implemented by constructing
the network $p(z|t)$ ($D_\text{group}$) to
be a Lie algebra parameterization $z = g(t)$
as defined in Eq. \ref{eq:lie_alg_def}.
In Fig. \ref{fig:lie_vae} (b) we show the details of how this module
is implemented.
For each input coordinate $t_j$ we learn a Lie algebra basis element
(a matrix) $A_j$ in shape $|V| \times |V|$
(recall that $V$ denotes the vector space on which the group is acting).
The coordinates and basis are aggregated and
fed into a matrix exponential mapping layer, resulting in a group
representation. Popular deep learning toolkits
TensorFlow \cite{tensorflow2015-whitepaper} and Pytorch
\cite{NEURIPS2019_9015} both offer in-built
differentiable implementations for matrix exponential map,
which are computed by approximation
methods proposed in \cite{tf_expm,pt_expm}.
We name a bottleneck-VAE equipped with a Lie group structure as a
Lie Group VAE.

Note that the bottleneck-VAE is essential to the realization of
learning equivariance with a group structure because the encoder network
needs to map the data directly onto a group structure
so that it can be encouraged to learn the correspondance between
the data variations and the group transformations.
If there is no such a feature-sharing constraint (using plain VAE
as a backbone), the encoder becomes a regular neural network
and is not trying to learn a equivariance
on the group structure but on the vector space, and the exponential
mapping layer is just to assist the decoder to
reconstruct the input data.



The one-parameter decomposition constraint and the
Hessian penalty constraint can be directly applied
as regularizations on Lie algebra basis $\{A_j\}_{j=1}^m$
(see Proposition \ref{prop:decomp} and \ref{prop:hessian}).
We name a Lie Group VAE equipped with either constraint a Commutative
Lie Group VAE as both constraints enforce a commutative
Lie group decomposition.

\section{Experiments}
\label{sec:experiments}
We conduct experiments by following the general
unsupervised disentanglement learning setup, i.e. training models
on a dataset without any supervision and evaluate the quality
of disentanglement by metrics on synthetic datasets and by
latent traversal inspection on real-world datasets.
Implementation details are shown in Appendix 5.
Code is available at
https://github.com/zhuxinqimac/CommutativeLieGroupVAE-Pytorch.
\subsection{Synthetic datasets}
We conduct experiments on the two most popular disentanglement datasets:
DSprites \cite{dsprites17} and 3DShapes \cite{Kim2018DisentanglingBF}.
The DSprites dataset consists of $64\times 64$ 2D shapes rendered with
5 independent generative factors, i.e. \emph{shape} (3 values),
\emph{scale} (6 values), \emph{orientation} (40 values),
\emph{x position} (32 values), and \emph{y position} (32 values).
The total number of data samples is 737,280, with each factor combination
appears exactly once.
The 3DShapes dataset contains $64\times 64$ images of 3D shapes generated from
6 independent factors, \emph{floor color} (10 values),
\emph{wall color} (10 values), \emph{object color} (10 values),
\emph{scale} (8 values), \emph{shape} (4 values),
\emph{orientation} (15 values).
There are 480,000 images in total in this dataset.
For evaluation, we report results with various disentanglement metrics
for an overall and robust evaluation.
The used metrics include FactorVAE metric (FVM) \cite{Kim2018DisentanglingBF},
SAP metric \cite{Kumar2017VariationalIO},
Mutual Information Gap (MIG) \cite{chen2018isolating},
and DCI Disentanglement metric \cite{Eastwood2018AFF}.
All the reported scores shown in this paper are averaged by 10 random runs.
For ablation studies, we randomly split the dataset into
training set (9/10) and test set (1/10) in each run,
and compute evaluation scores on the test set.
For the state-of-the-art comparison, we follow the tradition by
training and evaluating on the whole dataset.
We report the mean and standard deviation in all tables.

\begin{table}
    \small
    \begin{center}
        \begin{tabular}{lllll}
        \toprule
            Models & FVM & SAP & MIG & DCI \\
        \midrule
            VAE & $69.4_{\pm 10.9}$ & $19.7_{\pm 10.6}$ & $7.8_{\pm 6.4}$ & $8.1_{\pm 4.1}$ \\
            +bottle & $74.6_{\pm 8.1}$ & $29.2_{\pm 12.1}$ & $12.9_{\pm 6.6}$ & $11.6_{\pm 3.3}$ \\
            +exp & $\bsb{83.6}_{\pm 3.2}$ & $\bsb{40.7}_{\pm 12.2}$ & $\bsb{17.2}_{\pm 6.8}$ & $\bsb{15.1}_{\pm 2.4}$ \\
        \bottomrule
        \end{tabular}
    \end{center}
    \caption{Ablation study of bottleneck-VAE and exponential map on DSprites.}
    \label{table:components_dsprites}
\end{table}
\begin{figure}[t]
\begin{center}
   \includegraphics[width=\linewidth]{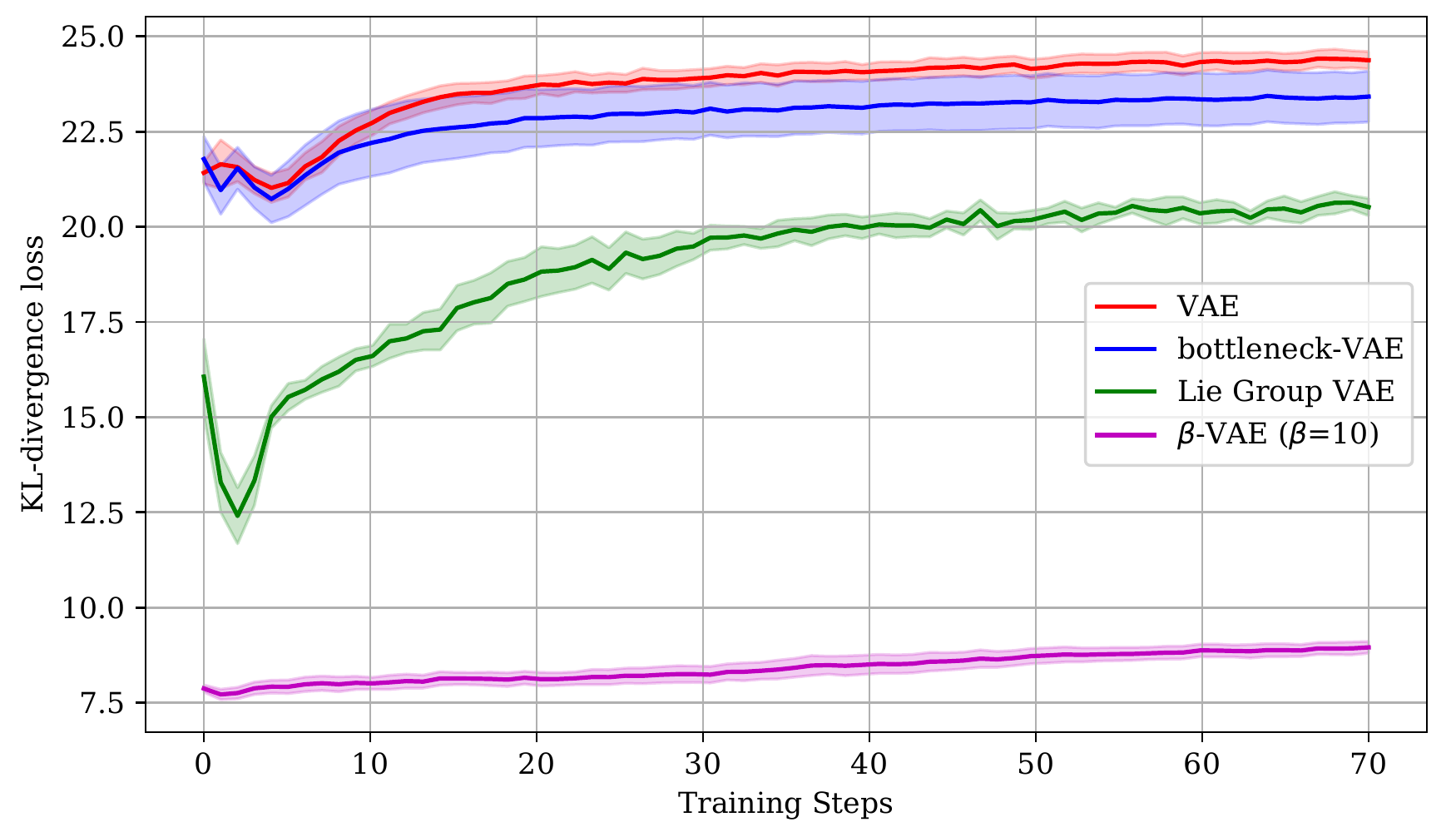}
\end{center}
    \caption{How the KL-divergence loss ($KL(q(t|x) || p(t))$)
    evolves during training for different models.}
\label{fig:kl_vs_iter}
\end{figure}
\textbf{Effectiveness of the Proposed Components.}
We first investigate if the proposed model components can contribute
to the disentanglement learning individually.
In Table \ref{table:components_dsprites}, we show how much a
VAE can gain from these incorporated modules
in terms of disentanglement scores. In Fig. \ref{fig:kl_vs_iter}
we show how the KL-divergence loss evolves.
In the table, the +bottle entry corresponds to the introduced
bottleneck-VAE which extends the plain VAE with an additional constraint
by enforcing a shared layer of features between the encoder and decoder.
Although its disentanglement results are still evidently inferior than
other complete models (e.g. see results in other tables), it is quite surprising
such a simple feature-sharing constraint can boost a VAE
by an obvious margin. A potential explanation is that the bottleneck-VAE
constrains the variations encoded in the latent codes
since the model should \emph{be very careful} to not change the
to-be-shared feature too much or it becomes harder to reconstruct.
This is beneficial to disentanglement since entanglement usually
comes from codes which capture too much information so that they
\emph{overlap} or \emph{intersect} with each other.
This constraint on variation encoding can also be observed
in Fig. \ref{fig:kl_vs_iter} where for the +bottleneck (blue) line
the KL-divergence between the posterior distribution and the latent prior
becomes slightly smaller than a plain VAE (red).
By adding the exponential mapping layer ($D_{group}$), which enforces
the model to keep a group structure in the bottleneck,
the disentanglement performance has been boosted to a competing level
with other state-of-the-art methods.
This is a key modification since from Fig. \ref{fig:kl_vs_iter}
we see the +exp (green) line evolves more elastically than
the VAE baselines, e.g. VAE, +bottleneck, $\beta$-VAE ($\beta$=10).
It shows that at the beginning the information encoded in the latent code
is heavily constrained (moving closer towards the $\beta$-VAE) by
the group-structure bottleneck. However because of
the adaptivity of the group structure, more variations can be
gradually learned as the training goes on.
On the contrary, the KL losses shown in other VAE baselines change
slowly, indicating the data variations are hard to be newly discovered
during training in these models.

\begin{table}
    \small
    \begin{center}
        \begin{tabular}{lllll}
        \toprule
            $\text{Size}_\text{group}$ & FVM & SAP & MIG & DCI \\
        \midrule
            4 & $23.6_{\pm 3.3}$ & $6.3_{\pm 6.0}$ & $4.2_{\pm 3.9}$ & $3_{\pm 0.5}$ \\
            9 & $57.4_{\pm 5.8}$ & $34.1_{\pm 12.9}$ & $17.3_{\pm 7.4}$ & $12.4_{\pm 4.4}$ \\
            25 & $79.8_{\pm 2.8}$ & $39.6_{\pm 13.4}$ & $20.6_{\pm 8.5}$ & $19.9_{\pm 3.8}$ \\
            64 & $82.7_{\pm 3.7}$ & $42.2_{\pm 12.5}$ & $22.1_{\pm 10.1}$ & $\bsb{20.0}_{\pm 6.8}$ \\
            81 & $84.4_{\pm 2.6}$ & $45.2_{\pm 10.5}$ & $23.0_{\pm 8.4}$ & $19.6_{\pm 6.3}$ \\
            100 & $\bsb{85.5}_{\pm 2.2}$ & $\bsb{50.8}_{\pm 5.0}$ & $\bsb{25.4}_{\pm 6.1}$ & $19.7_{\pm 4.6}$ \\
        \bottomrule
        \end{tabular}
    \end{center}
    \caption{Ablation study of group size on DSprites.}
    \label{table:group_size_dsprites}
\end{table}
\begin{table}
    \small
    \begin{center}
        \begin{tabular}{lllll}
        \toprule
            $\lambda_\text{decomp}$ & FVM & SAP & MIG & DCI \\
        \midrule
            0 & $83.6_{\pm 3.2}$ & $40.7_{\pm 12.2}$ & $17.2_{\pm 6.8}$ & $15.1_{\pm 2.4}$ \\
            5 & $84.0_{\pm 3.9}$ & $45.4_{\pm 11.5}$ & $20.5_{\pm 6.9}$ & $16.8_{\pm 4.3}$ \\
            20 & $\bsb{85.8}_{\pm 6.9}$ & $48.7_{\pm 8.4}$ & $23.6_{\pm 5.0}$ & $18.2_{\pm 3.0}$ \\
            40 & $85.5_{\pm 2.2}$ & $\bsb{50.8}_{\pm 5.0}$ & $\bsb{25.4}_{\pm 6.1}$ & $\bsb{19.7}_{\pm 4.6}$ \\
            80 & $85.5_{\pm 4.8}$ & $47.1_{\pm 8.6}$ & $23.3_{\pm 6.2}$ & $18.3_{\pm 6.5}$ \\
        \bottomrule
        \end{tabular}
    \end{center}
    \caption{Ablation study of one-parameter decomposition on DSprites.}
    \label{table:one_param_dsprites}
\end{table}
\textbf{How the Group Representation Size Affects Disentanglement.}
In Table \ref{table:group_size_dsprites} we show how the disentanglement
scores are affected by the choice of group representation size.
Recall that the groups are represented by invertible matrices
(subgroups of $GL(V)$) therefore they all have sizes of squared numbers.
It can be expected that the larger the group representation is,
the more likely the subgroups (indexed by coordinates $t_i$'s) can control
different variations (due to the increased sparsity in a larger space).
In the table, we can see the group representation of size $4$ has almost
no capability for disentanglement because a group of $2\times 2$
matrices represents the linear transformations on a 2D plane,
which is hard to be decomposed into more than two
independent sub-transformations (e.g. scaling + rotation).
When the representation size exceeds 25 ($5 \times 5$), the model can
easily find a decomposition to represent different data variations,
and the disentanglement scores saturate.
We use the group representation size 100 for all other experiments.

\begin{table}
    \small
    \begin{center}
        \begin{tabular}{lllll}
        \toprule
            $\lambda_\text{hessian}$ & FVM & SAP & MIG & DCI \\
        \midrule
            0 & $83.6_{\pm 3.2}$ & $40.7_{\pm 12.2}$ & $17.2_{\pm 6.8}$ & $15.1_{\pm 2.4}$ \\
            5 & $83.8_{\pm 2.4}$ & $46.8_{\pm 12.8}$ & $19.8_{\pm 8.6}$ & $17.5_{\pm 5.6}$ \\
            20 & $86.1_{\pm 1.8}$ & $\bsb{54.1}_{\pm 1.2}$ & $\bsb{29.7}_{\pm 3.1}$ & $\bsb{23.4}_{\pm 4.1}$ \\
            40 & $\bsb{86.2}_{\pm 1.8}$ & $48.2_{\pm 1.9}$ & $25.2_{\pm 8.4}$ & $19.1_{\pm 4.1}$ \\
            80 & $85.0_{\pm 1.6}$ & $43.6_{\pm 11.3}$ & $20.1_{\pm 8.4}$ & $17.4_{\pm 4.2}$ \\
        \bottomrule
        \end{tabular}
    \end{center}
    \caption{Ablation study of Hessian penalty on DSprites.}
    \label{table:hessian_dsprites}
\end{table}
\begin{figure}[t]
\begin{center}
   \includegraphics[width=\linewidth]{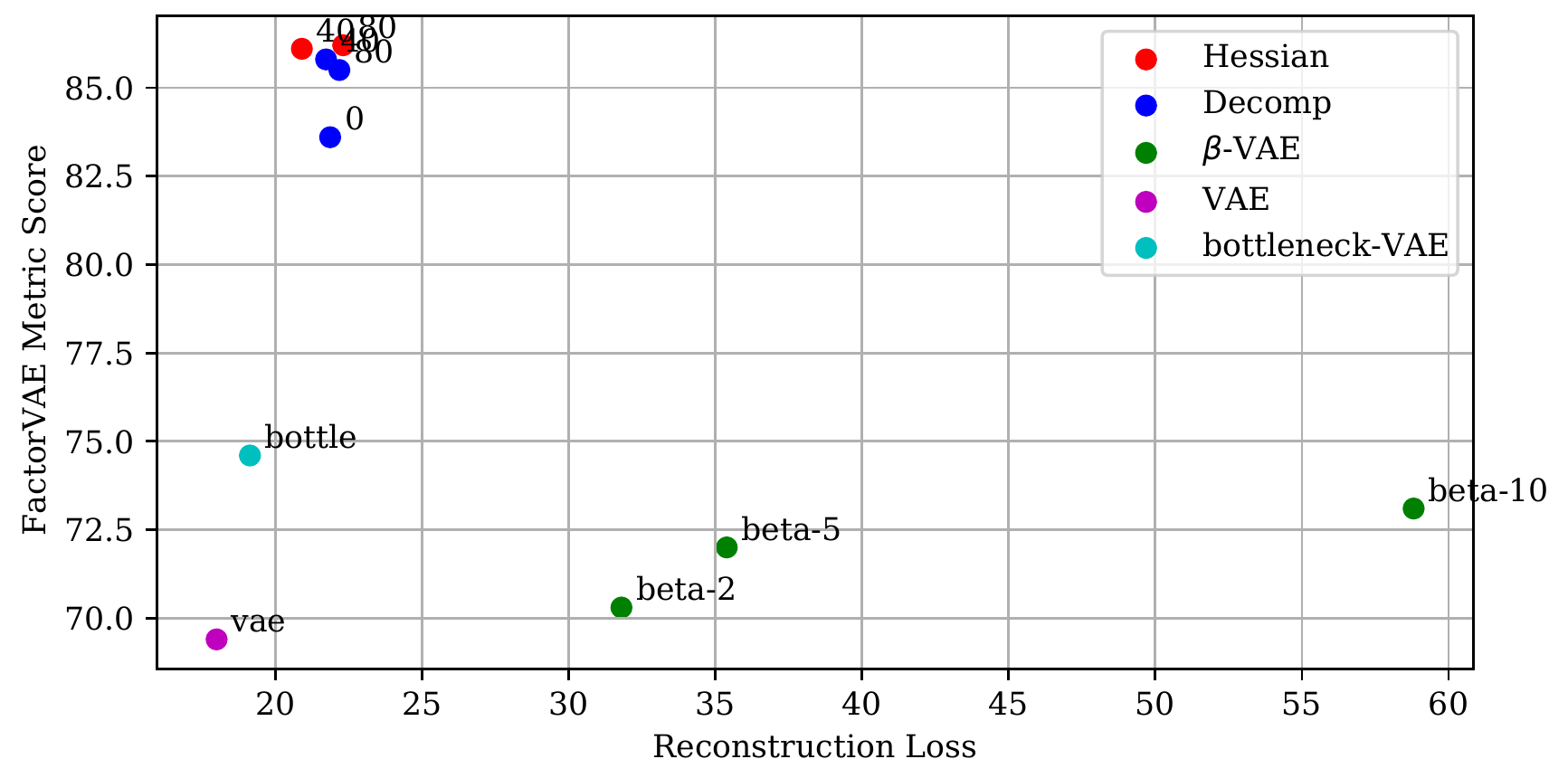}
\end{center}
    \caption{Reconstruction loss vs FactorVAE metric.}
\label{fig:recons_vs_fvm}
\end{figure}
\textbf{Effectiveness of One-parameter Subgroup Decomposition and
Hessian Penalty Constraints.}
Now we investigate how the induced Lie-group related constraints
benefit disentanglement learning.
For each constraint we use a hyper-parameter
$\lambda_{decomp} (\lambda_{hessian})$ to modulate the effect.
The results of one-parameter subgroup decomposition is shown in
Table \ref{table:one_param_dsprites}, and Hessian Penalty is in
Table \ref{table:hessian_dsprites}.
Both constraints can benefit disentanglement performance,
where $\lambda_{decomp}$ reaches the peak at 40 and $\lambda_{hessian}$
at 20.
We can see the Hessian Penalty constraint is more effective
than the subgroup decomposition as the performance gain in the former
one is more significant (54.1 vs 50.8 on SAP and 29.7 vs 25.4 on MIG).
This is due to that Hessian Penalty is a stronger constraint than
the subgroup decomposition since it requires the Lie algebra
basis elements to have mutual products of zeros while subgroup decomposition
only requires their commutators to be zeros.
This proves that forcing different subgroups to have independent
effect on the final group representation is very beneficial to
capture factorized data variations.
In Fig. \ref{fig:recons_vs_fvm} we show the scatter plot of
FactorVAE metric against reconstruction loss.
We see our proposed Commutative Lie Group constraints
boost the disentanglement performance at slight cost of
reconstruction quality, while the $\beta$-VAEs sacrifice
reconstruction severely.

\begin{table}[t]
    \begin{center}
        \begin{tabular}{lll}
        \toprule
            Model & DSprites & 3DShapes \\
        \midrule
            VAE & $69.4_{\pm 10.9}$ & $83.6_{\pm 6.5}$ \\
            $\beta$-VAE & $74.4_{\pm 7.7}$ & $91$ \cite{Kim2018DisentanglingBF} \\
            Cascade-VAE & $81.74_{\pm 2.97}$ & - \\
            Factor-VAE & $82.15_{\pm 0.88}$ & $89$ \cite{Kim2018DisentanglingBF} \\
        \midrule
            Ours & $\bsb{86.1}_{\pm 2.0}$ & $\bsb{93.2}_{\pm 4.0}$ \\
        \bottomrule
        \end{tabular}
    \end{center}
    \caption{Unsupervised disentanglement state-of-the-art comparison on
    DSprites and 3DShapes.}
    \label{table:unsupervised_SOTA}
\end{table}
\begin{figure}[t]
\begin{center}
   \includegraphics[width=\linewidth]{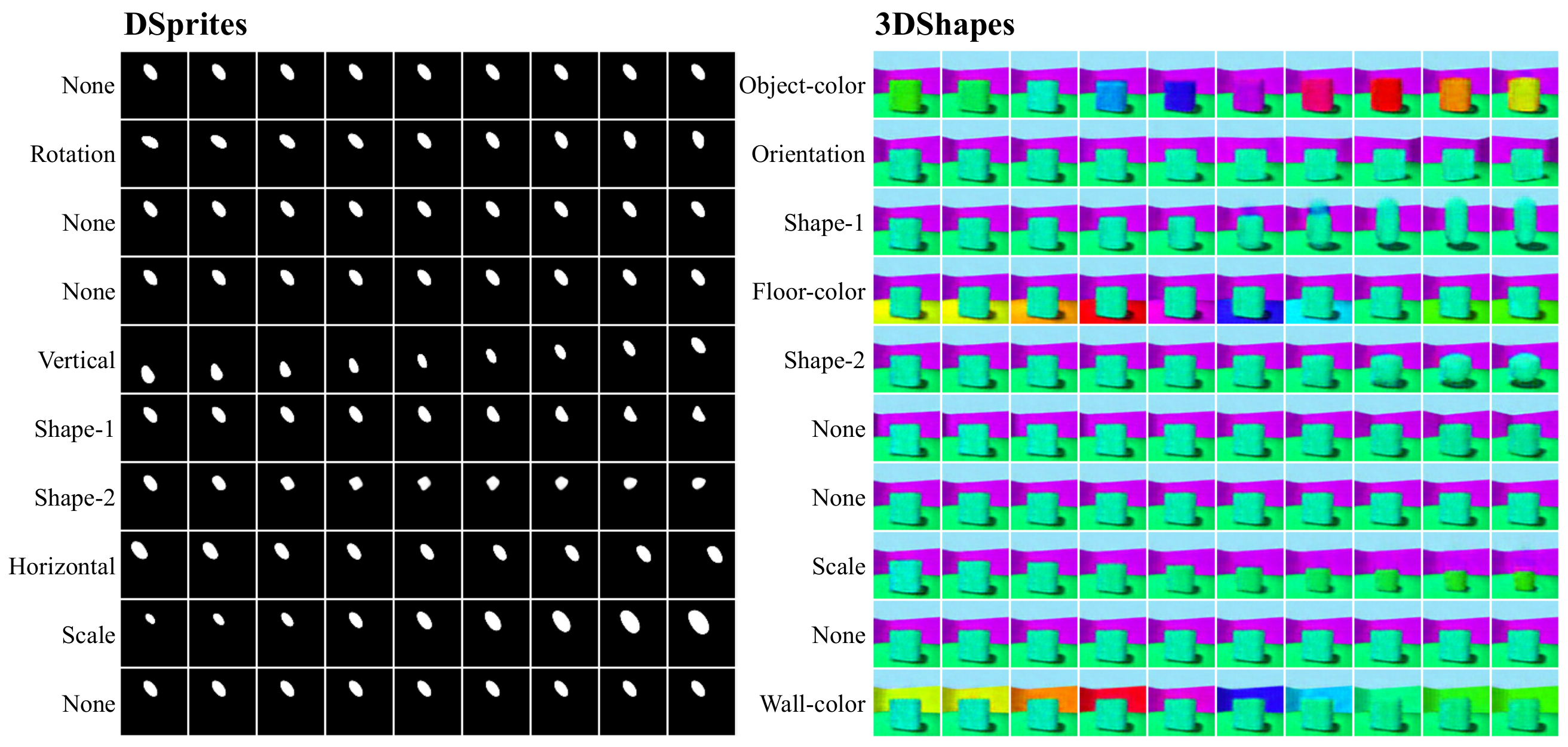}
\end{center}
    \caption{Latent traversals of our Commutative Lie Group VAE on
    DSprites and 3DShapes datasets.}
\label{fig:dsp_3ds_traversals}
\end{figure}
\textbf{State-of-the-art Comparison.}
In Table \ref{table:unsupervised_SOTA} we compare our Commutative
Lie Group VAEs
with other state-of-the-art models for learning continuous
latent variables on DSprites and 3DShapes datasets.
We use the whole dataset to train and evaluate our models
following the tradition
of unsupervised disentanglement learning.
Since most of the compared models only report the FVM score,
we follow this standard and report the best model on FVM.
We can see our Lie Group VAE model achieves the best performance
among all compared models.
It should be noticed that all other compared models enforce
the statistical independence assumption to achieve disentanglement
while our model is built based on another completely assumption
by leveraging group structures.
This indicates that our model has the potential to be further
improved if independence assumption is concurrently enforced.
Qualitative latent traversal results of our Commutative Lie Group VAE on
both datasets are shown in Fig. \ref{fig:dsp_3ds_traversals}.

\subsection{Real-world Datasets}
In this section we run our Commutative Lie Group VAE on real-world datasets
including CelebA \cite{Liu2014DeepLF}, Mnist \cite{MNIST},
and 3DChairs \cite{Aubry2014Seeing3C}.

\textbf{CelebA} dataset contains 202,599 images of cropped real-world
human faces. We crop the center
$128 \times 128$ area and resize the images to
$64 \times 64$ for this experiment.
In Fig. \ref{fig:celeba_trav} we show the qualitative results of our
model ($\lambda_{hessian} = 40$) trained on CelebA, and compare
it with the FactorVAE baseline.
For most of the attributes, our model can extract cleaner semantic variations
than FactorVAE. For example, the \emph{background} concept captured by
FactorVAE is entangled with \emph{smile} while in our model
it is independently encoded.
Additionally, our model learns to encode the semantics of
\emph{forehead hair-style} and \emph{make-up} which are not shown
in the FactorVAE.
\begin{figure}[t]
\begin{center}
   \includegraphics[width=\linewidth]{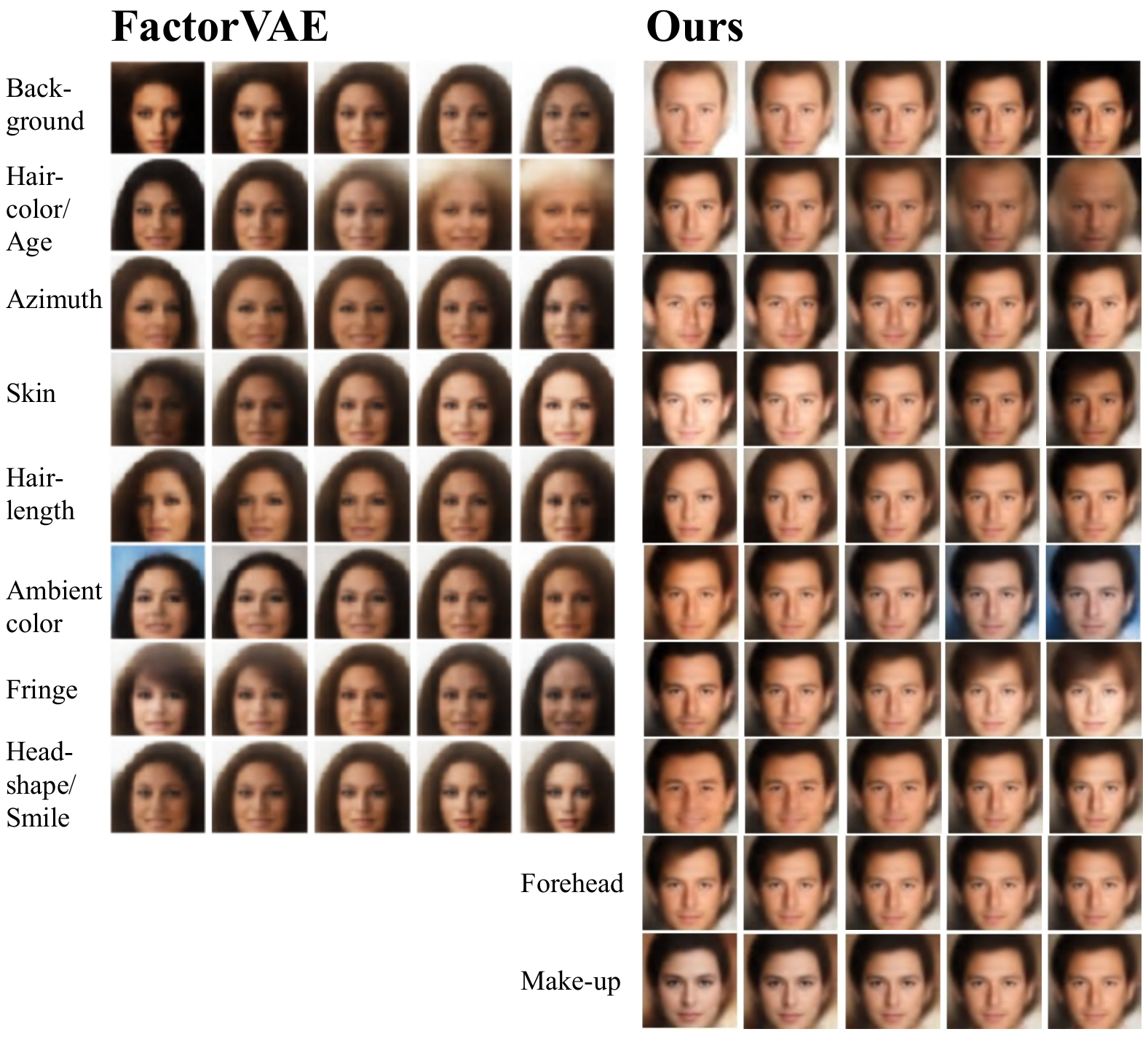}
\end{center}
    \caption{CelebA traversals compared with FactorVAE
    \cite{Kim2018DisentanglingBF}.}
\label{fig:celeba_trav}
\end{figure}

\textbf{Mnist} dataset consists of handwritten digits ($28\times 28$
images) in 10 classes.
We pad the images to size $32\times 32$ for easier usage.
We train our model with data of a same class
to learn continuous variations.
It is possible to integrate techniques for learning discrete latent
variables \cite{Dupont2018LearningDJ,Jeong2019LearningDA} in our
Commutative Lie Group VAE for unsupervised classification,
but we leave it for future work.
In Fig. \ref{fig:mnist_trav}, we show some interesting
semantic variations discovered by our model.
We can see our model learns to control the circle size in \emph{six}
and \emph{nine} with a variable, and also learns to represent
the subtle variation of curviness in \emph{one}.
In \emph{four}, our model discovers a \emph{lifting} concept which
controls the level of the horizontal line shown in 4's.
\begin{figure}[t]
\begin{center}
   \includegraphics[width=\linewidth]{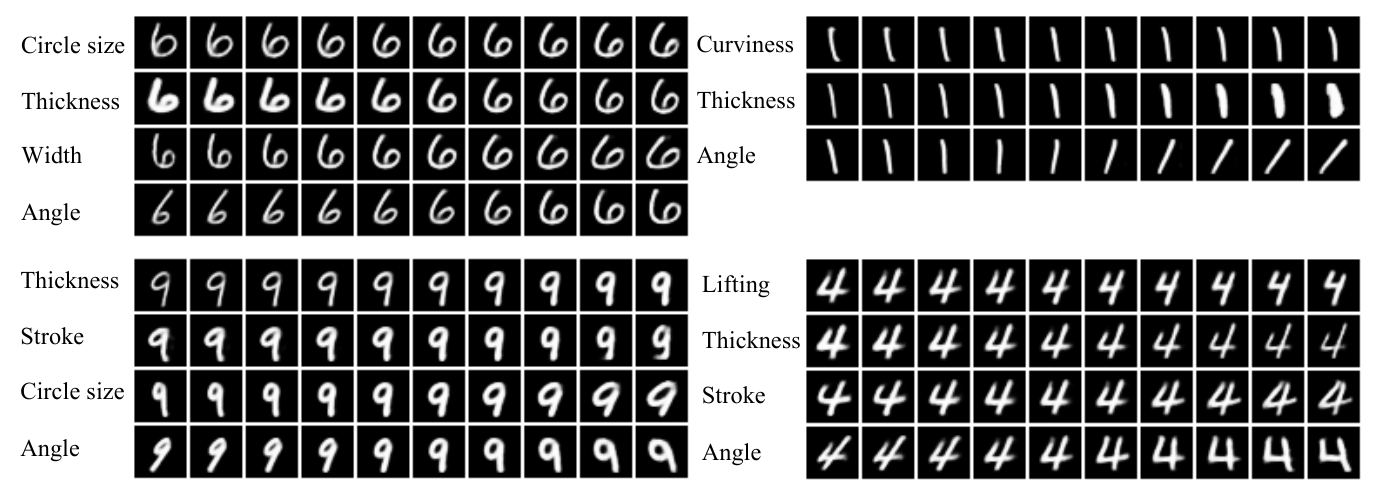}
\end{center}
    \caption{Per-class latent traversals on Mnist dataset.}
\label{fig:mnist_trav}
\end{figure}

\textbf{3DChairs} dataset contains 86,366 RGB images of various chairs of
resolution $64 \times 64$.
In Fig. \ref{fig:chairs_trav}, we compare our model
with the state-of-the-art CascadeVAE \cite{Jeong2019LearningDA}.
Both models achieve similar disentanglement quality, showing that
our model though based on a different assumption of group theory,
can still achieve the same-level
results as the state-of-the-art CascadeVAE based on information theory
which models statistical independence.
\begin{figure}[t]
\begin{center}
   \includegraphics[width=\linewidth]{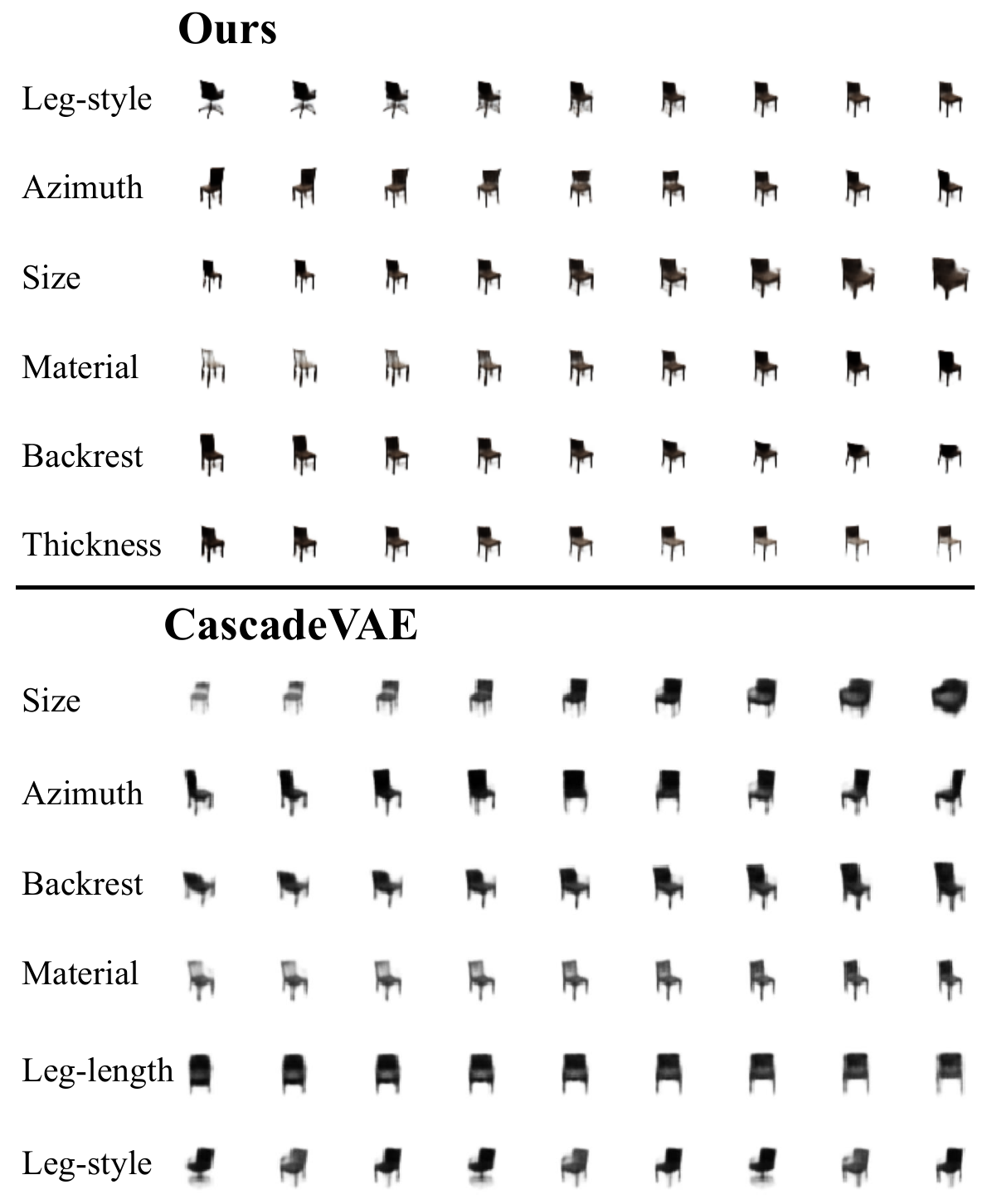}
\end{center}
    \caption{Comparing latent traversals of Commutative Lie Group VAE
    with CascadeVAE on 3DChairs.}
\label{fig:chairs_trav}
\end{figure}

\section{Conclusion}
In this paper we proposed to learn disentangled representations without
supervision by
capturing the data variations with an adaptive group structure.
This is based on the idea that a group structure can represent data variations
by group actions applied to itself.
The advantage of using groups over usual vector spaces is that
a group structure can not only equivariantly represent
variations but also be adaptively optimized to fit the diversity
in data variations.
By replacing general vector representations with group representations,
we can represent a data point with a group element on a group structure.
To enable practical training, we adopted the Lie algebra parameterization
and converted the learning problem on groups into the learning in linear
spaces, which enables general optimization and sampling.
Based on the parameterization, we introduced two simple commutative
decomposition constraints for
encouraging disentanglement, which are naturally derived from
the one-parameter subgroup decomposition assumption and the Hessian
Penalty assumption.
To instantiate the group-based learning method, we introduced a
variant of VAE called bottleneck-VAE, derived from a new lower bound
of the data log-likelihood.
We then proposed our (Commutative) Lie Group VAE by simply
integrating an exponential mapping
layer (with commutative decomposition constraints)
into the decoder of the bottleneck-VAE.
The proposed model achieved state-of-the-art performance in
unsupervised disentanglement learning without adopting other regular
constraints like statistical independence.
Our proposed method is simple, elegant, and effective, and we believe
this model exhibits a new direction for learning disentangled representations.

\section*{Acknowledgement}
This work was supported in part by the Australian
Research Council under Projects DE180101438 and DP210101859.


\bibliography{example_paper}
\bibliographystyle{icml2021}

\end{document}